\def\year{2021}\relax
\documentclass[letterpaper]{article} 
\usepackage{aaai21}  
\usepackage{times}  
\usepackage{helvet} 
\usepackage{courier}  
\usepackage[hyphens]{url}  
\usepackage{graphicx} 
\usepackage{natbib}  
\usepackage{caption} 

\usepackage{amsmath}
\usepackage{amsthm}
\usepackage{amsfonts}   
\usepackage[ruled]{algorithm2e}
\usepackage{footnote}
\usepackage{wrapfig}
\usepackage{subfig}
\usepackage{grffile}
\usepackage{mathtools}
\usepackage{lineno}
\DontPrintSemicolon
\mathtoolsset{showonlyrefs}

\newcounter{lem_counter}
\newcounter{pro_counter}

\newcounter{ass_counter}

\newcounter{rmk_counter}

\newtheorem{proposition}[pro_counter]{Proposition}
\newtheorem{lemma}[lem_counter]{Lemma}

\newtheorem{assumption}[ass_counter]{Assumption}

\newtheorem{remark}[rmk_counter]{Remark}

\newcommand{\E}{\mathbb{E}}
\newcommand{\R}{\mathbb{R}}
\newcommand{\V}{\mathbb{V}}

\title{Mean-Variance Policy Iteration for Risk-Averse Reinforcement Learning}
\author {
        Shangtong Zhang,\textsuperscript{\rm 1}
        Bo Liu, \textsuperscript{\rm 2}
        Shimon Whiteson \textsuperscript{\rm 1} \\
}
\affiliations {
    \textsuperscript{\rm 1} University of Oxford \\
    \textsuperscript{\rm 2} Auburn University \\
    shangtong.zhang@cs.ox.ac.uk, bo.liu@auburn.edu, shimon.whiteson@cs.ox.ac.uk
}
\begin{document}

\maketitle

\begin{abstract}
We present a mean-variance policy iteration (MVPI) framework for risk-averse control in a discounted infinite horizon MDP optimizing the variance of a per-step reward random variable.
MVPI enjoys \textit{great flexibility} in that
any policy evaluation method and risk-neutral control method can be dropped in for risk-averse control off the shelf,
in both on- and off-policy settings.
This flexibility reduces the gap between risk-neutral control and risk-averse control and is achieved by working on a novel augmented MDP directly.
We propose risk-averse TD3 as an example instantiating MVPI,
which outperforms vanilla TD3 and many previous risk-averse control methods in challenging Mujoco robot simulation tasks under a risk-aware performance metric.
This risk-averse TD3 is the first to introduce deterministic policies and off-policy learning into risk-averse reinforcement learning,
both of which are key to the performance boost we show in Mujoco domains.
\end{abstract}

\section{Introduction}
\label{sec:intro}
One fundamental task in reinforcement learning (RL, \citealt{sutton2018reinforcement}) is control,
in which we seek a policy that maximizes certain performance metrics.
In risk-neutral RL,
the performance metric is usually the expectation of some random variable,
for example, the expected total (discounted or undiscounted) reward \citep{puterman2014markov,sutton2018reinforcement}.
We, however, 
sometimes want to minimize certain risk measures of that random variable while maximizing its expectation. 
For example,
a portfolio manager usually wants to reduce the risk of a portfolio while maximizing its return.
Risk-averse RL is a framework for studying such problems.

Although many real-world applications can potentially benefit from risk-averse RL, e.g., pricing \citep{wang2000class}, 
healthcare \citep{parker2009managing}, 
portfolio management \citep{lai2011mean},  autonomous driving \citep{maurer2016autonomous}, and robotics \citep{majumdar2020should},
the development of risk-averse RL largely falls behind risk-neutral RL.
Risk-neutral RL methods have enjoyed superhuman performance in many domains, e.g., Go \citep{silver2016mastering}, protein design \citep{senior2018alphafold}, and StarCraft II \citep{vinyals2019grandmaster},
while no human-level performance has been reported for risk-averse RL methods in real-world applications.
Risk-neutral RL methods have enjoyed stable off-policy learning \citep{watkins1992q,maei2011gradient,fujimoto2018addressing,haarnoja2018soft},
while state-of-the-art risk-averse RL methods, e.g., \citet{liu2018block,papini:risk2019}, still require on-policy samples.
Risk-neutral RL methods have exploited deep neural network function approximators and distributed training \citep{mnih2016asynchronous,espeholt2018impala},
while tabular and linear methods still dominate the experiments of risk-averse RL literature \citep{tamar2012policy,prashanth2013actor,liu2018block,chow2018risk}.
Such a big gap between risk-averse RL and risk-neutral RL gives rise to a natural question:
\emph{can we design a meta algorithm that can easily leverage recent advances in risk-neutral RL for risk-averse RL?}
In this paper, we give an affirmative answer via the mean-variance policy iteration (MVPI) framework.

Although many risk measures have been used in risk-averse RL,
in this paper, 
we mainly focus on variance \citep{sobel1982variance,mannor2011mean,tamar2012policy,prashanth2013actor,liu2018block} given its advantages in interpretability and computation \citep{markowitz2000mean,li2000optimal}.
Such an RL paradigm is usually referred to as mean-variance RL,
and previous mean-variance RL methods usually consider the variance of the total reward random variable \citep{tamar2012policy,prashanth2013actor,liu2018block}.
Recently,
\citet{papini:risk2019} propose a reward-volatility risk measure that considers the variance of a per-step reward random variable.
\citet{papini:risk2019} show that the variance of the per-step reward can better capture the short-term risk than the variance of the total reward and usually leads to smoother trajectories.

In complicated environments with function approximation,
pursuing the exact policy that minimizes the variance of the total reward is usually intractable.
In practice, 
all we can hope is to lower the variance of the total reward to a certain level.
As the variance of the per-step reward bounds the variance of the total reward from above \citep{papini:risk2019},
in this paper, 
we propose to optimize the variance of the per-step reward as a proxy for optimizing the variance of the total reward.
Though the policy minimizing the variance of the per-step reward does not necessarily minimize the variance of the total reward,
in this paper,
we show, 
via MVPI,
that optimizing the variance of the per-step reward as a proxy is more efficient and scalable than optimizing the variance of the total reward directly.

MVPI enjoys great flexibility in that
\emph{any policy evaluation method and risk-neutral control method can be dropped in for risk-averse control off the shelf,
in both on- and off-policy settings}.
Key to the flexibility of MVPI is 
that it works on an augmented MDP directly,
which we make possible 
by introducing the Fenchel duality and block cyclic coordinate ascent to solve a policy-dependent reward issue \citep{papini:risk2019}.
This issue refers to a requirement to solve an MDP whose reward function depends on the policy being followed,
i.e.,
the reward function of this MDP is nonstationary.
Consequently,
standard tools from the MDP literature are not applicable.
We propose risk-averse TD3 as an example instantiating MVPI,
which outperforms vanilla TD3 \citep{fujimoto2018addressing} and many previous mean-variance RL methods \citep{tamar2012policy,prashanth2013actor,liu2018block,papini:risk2019} in challenging Mujoco robot simulation tasks in terms of a risk-aware performance metric.
To the best of our knowledge, 
we are the first to benchmark mean-variance RL methods in Mujoco domains,
a widely used benchmark for robotic-oriented RL research,
and the first to bring off-policy learning and deterministic policies into mean-variance RL.

\section{Mean-Variance RL}
\label{sec:bg}
We consider an infinite horizon MDP with a state space $\mathcal{S}$, an action space $\mathcal{A}$, 
a bounded reward function $r: \mathcal{S} \times \mathcal{A} \rightarrow \R$,
a transition kernel $p: \mathcal{S} \times \mathcal{S} \times \mathcal{A} \rightarrow [0, 1]$, an initial distribution $\mu_0: \mathcal{S} \rightarrow [0, 1]$, and a discount factor $\gamma \in [0, 1]$.
The initial state $S_0$ is sampled from $\mu_0$.
At time step $t$, an agent takes an action $A_t$ according to $\pi(\cdot | S_t)$, where $\pi: \mathcal{A} \times \mathcal{S} \rightarrow [0, 1]$ is the policy followed by the agent.
The agent then gets a reward $R_{t+1} \doteq r(S_t, A_t)$ and proceeds to the next state $S_{t+1}$ according to $p(\cdot | S_t, A_t)$.
In this paper, we consider a deterministic reward setting for the ease of presentation, following \citet{chow2017risk,liu2018block}.
The return at time step $t$ is defined as
$G_t \doteq \sum_{i=0}^{\infty} \gamma^{i} r(S_{t+i}, A_{t+i})$.
When $\gamma < 1$, 
$G_t$ is always well defined.
When $\gamma = 1$, 
to ensure $G_t$ remains well defined,
it is usually assumed that all polices are proper \citep{bertsekas1996neuro},
i.e.,
for any policy $\pi$, 
the chain induced by $\pi$ has some absorbing states, one of which the agent will eventually go to with probability 1.
Furthermore, the rewards are always 0 thereafter. 
For any $\gamma \in [0, 1]$, $G_0$ is the random variable indicating the total reward,
and we use its expectation 
\begin{align}
J(\pi) \doteq \E_{\mu_0, p, \pi}[G_0],
\end{align}
as our primary performance metric.
In particular, when $\gamma = 1$, we can express $G_0$ as $G_0 = \sum_{t=0}^{T-1} r(S_t, A_t)$,
where $T$ is a random variable indicating the first time the agent goes to an absorbing state.
For any $\gamma \in [0, 1]$, 
the state value function and the state-action value function are defined as $v_\pi(s) \doteq \E[G_t | S_t = s]$
and $q_\pi(s, a) \doteq \E[G_t | S_t = s, A_t = a]$ respectively.

\textbf{Total Reward Perspective.}
Previous mean-variance RL methods \citep{prashanth2013actor,tamar2012policy,liu2018block} usually consider the variance of the total reward.
Namely, they consider the following problem:
\begin{align}
\label{pb:original}
\textstyle{\max_\theta \E[G_0] \quad \text{subject to} \quad \V(G_0) \leq \xi},
\end{align}
where $\V(\cdot)$ indicates the variance of a random variable, $\xi$ indicates the user's tolerance for variance, and $\pi$ is parameterized by $\theta$. 
In particular, \citet{prashanth2013actor} consider the setting $\gamma < 1$ and convert \eqref{pb:original} into an unconstrained saddle-point problem:
$\textstyle{\max_\lambda \min_\theta L_1(\theta, \lambda) \doteq -\E[G_0] + \lambda (\V(G_0) - \xi)}$,
where $\lambda$ is the dual variable.
\cite{prashanth2013actor} use stochastic gradient descent to find the saddle-point of $L_1(\theta, \lambda)$.
To estimate $\nabla_{\theta, \lambda} L_1(\theta, \lambda)$, 
they propose two simultaneous perturbation methods:
simultaneous perturbation stochastic approximation and smoothed functional \citep{Bhatnagar_2013},
yielding a three-timescale algorithm.
Empirical success is observed in a simple traffic control MDP.

\citet{tamar2012policy} consider the setting $\gamma = 1$.
Instead of using the saddle-point formulation in \citet{prashanth2013actor},
they consider the following unconstrained problem: 
$\textstyle{\max_\theta L_2(\theta) \doteq \E[G_0] - \lambda g(\V(G_0) - \xi)}$,
where $\lambda > 0$ is a hyperparameter to be tuned and $g(\cdot)$ is a penalty function,
which they define as
$g(x) \doteq (\max\{0, x\})^2$.
The analytical expression of $\nabla_\theta L_2(\theta)$ they provide involves a term $\E[G_0] \nabla_\theta \E[G_0]$.
To estimate this term,
\citet{tamar2012policy} consider a two-timescale algorithm and keep running estimates for $\E[G_0]$ and $\V[G_0]$ in a faster timescale, yielding an episodic algorithm. 
Empirical success is observed in a simple portfolio management MDP.

\citet{liu2018block} consider the setting $\gamma = 1$ and set the penalty function $g(\cdot)$ in \citet{tamar2012policy} to the identity function.
With the Fenchel duality $x^2 = \max_y (2xy - y^2)$,
they transform the original problem into
$\textstyle{\max_{\theta, y} L_3(\theta, y) \doteq  2y(\E[G_0] + \frac{1}{2\lambda}) - y^2 - \E[G_0^2]}$,
where $y$ is the dual variable.
\citet{liu2018block} then propose a solver based on stochastic coordinate ascent,
yielding an episodic algorithm.



\textbf{Per-Step Reward Perspective.}
Recently \citet{papini:risk2019} propose a reward-volatility risk measure, 
which is the variance of a per-step reward random variable $R$.
In the setting $\gamma < 1$,
it is well known that the expected total discounted reward can be expressed as
\begin{align}
\textstyle{J(\pi) = \frac{1}{1 - \gamma} \sum_{s, a} d_\pi(s, a) r(s, a)},
\label{eq:Jpi_disc}
\end{align}
where $d_\pi(s, a)$ is the normalized discounted state-action distribution: 
\begin{align}
\textstyle{d_\pi(s, a) \doteq (1 - \gamma) \sum_{t=0}^\infty \gamma^t \Pr(S_t = s, A_t = a | \mu_0, \pi, p)}.
\end{align}
We now define the per-step reward random variable $R$, 
a discrete random variable taking values in the image of $r$,
by defining its probability mass function as $p(R = x) = \sum_{s, a} d_\pi(s, a)\mathbb{I}_{r(s, a) = x}$,
where $\mathbb{I}$ is the indicator function. 
It follows that
$\textstyle{\E[R] = (1 - \gamma) J(\pi)}$.
\citet{papini:risk2019} argue that $\V(R)$ can better capture short-term risk than $\V(G_0)$ and optimizing $\V(R)$ usually leads to smoother trajectories than optimizing $\V(G_0)$, among other advantages of this risk measure.
\citet{papini:risk2019}, therefore, consider the following objective:
\begin{align}
\label{eq:mvpi-obj}
J_\lambda(\pi) \doteq \E[R] - \lambda \V(R).
\end{align}
\citet{papini:risk2019} show that $J_\lambda(\pi) = \E[R- \lambda (R-\E[R])^2]$,
i.e., to optimize the risk-aware objective $J_\lambda(\pi)$ is to optimize the canonical risk-neutral objective of a new MDP, 
which is the same as the original MDP except that the new reward function is $$r^\prime(s, a) \doteq r(s, a) - \lambda \big( r(s, a) - (1 - \gamma)J(\pi) \big)^2.$$
Unfortunately,
this new reward function $r^\prime$ depends on the policy $\pi$ due to the occurrence of $J(\pi)$,
implying the reward function is actually nonstationary.
By contrast,
in canonical RL settings (e.g., \citet{puterman2014markov,sutton2018reinforcement}),
the reward function is assumed to be stationary.
We refer to this problem as the \emph{policy-dependent-reward} issue.
Due to this issue, the rich classical MDP toolbox cannot be applied to this new MDP easily, and the approach of \citet{papini:risk2019} \emph{does not and cannot} work on this new MDP directly. 

\citet{papini:risk2019} instead work on the objective Eq~\eqref{eq:mvpi-obj} directly \emph{without} resorting to the augmented MDP.
They propose to optimize a performance lower bound of $J_\lambda(\pi)$ by extending the performance difference theorem (Theorem 1 in \citet{schulman2015trust}) from the risk-neutral objective $J(\pi)$ to the risk-aware objective $J_\lambda(\pi)$,
yielding the Trust Region Volatility Optimization (TRVO) algorithm,
which is similar to Trust Region Policy Optimization  \citep{schulman2015trust}.

Importantly, \citet{papini:risk2019} show that $\V(G_0) \leq \frac{\V(R)}{(1 - \gamma)^2}$, 
indicating that minimizing the variance of $R$ implicitly minimizes the variance of $G_0$. 
We, 
therefore,
can optimize $\V(R)$ as a proxy (upper bound) for optimizing $\V(G_0)$.
In this paper, we argue that $\V(R)$ is easier to optimize than $\V(G_0)$. 
The methods of \citet{tamar2012policy,liu2018block} optimizing $\V(G_0)$ involve terms like $(\E[G_0])^2$ and $\E[G_0^2]$,
which lead to terms like $G_0^2\sum_{t=0}^{T-1} \nabla_\theta \log \pi(A_t | S_t)$ in their update rules,
yielding large variance.
In particular, it is computationally prohibitive to further expand $G_0^2$ explicitly to apply variance reduction techniques like baselines \citep{williams1992simple}.
By contrast, we show in the next section that by considering $\V(R)$, MVPI involves only $r(s, a)^2$,
which is easier to deal with than $G_0^2$.

\section{Mean-Variance Policy Iteration}
Although in many problems our goal is to maximize the expected total undiscounted reward,
practitioners often find that optimizing the discounted objective ($\gamma < 1$) as a proxy for the undiscounted objective ($\gamma = 1$) is  better than optimizing the undiscounted objective directly,
especially when deep neural networks are used as function approximators \citep{mnih2015human,lillicrap2015continuous,espeholt2018impala,xu2018meta,van2019using}.
We, therefore, focus on the discounted setting in the paper,
which allows us to consider optimizing the variance of the per-step reward as a proxy (upper bound) for optimizing the variance of the total reward.

To address the policy-dependent reward issue,
we use the Fenchel duality to rewrite $J_\lambda(\pi)$ as
\begin{align}
\label{eq:obj}
J_\lambda(\pi) &= \E[R] - \lambda \E[R^2] + \lambda (\E[R])^2\\
&= \E[R] - \lambda \E[R^2] + \lambda \max_y \big( 2\E[R]y - y^2 \big),
\end{align} 
yielding the following problem:
\begin{align}
\label{pb:per-step-reward}
&\max_{\pi, y} J_\lambda(\pi, y) \\
\doteq &\textstyle{\sum_{s, a} d_\pi(s, a) \big(r(s, a) - \lambda r(s, a)^2 + 2\lambda r(s, a) y \big) - \lambda y^2}.
\end{align}
We then propose a \emph{block cyclic coordinate ascent} (BCCA, \citealt{luenberger1984linear,tseng2001convergence,saha2010finite,saha2013nonasymptotic,wright2015coordinate})
framework to solve \eqref{pb:per-step-reward},
which updates $y$ and $\pi$ alternatively as shown in Algorithm~\ref{alg:mvpi}.
\begin{algorithm}
\For{k = 1, \dots}{
\textbf{Step 1}: $y_{k+1} \doteq (1 - \gamma) J(\pi_k)$ \tcp{The exact solution for $\arg {\max _y}J_\lambda({\pi_k},y)$} 
\textbf{Step 2}: $\pi_{k + 1} \doteq \arg\max_\pi \Big( \sum_{s, a} d_\pi(s, a) \big(r(s,a) - \lambda r{(s,a)^2} + 2\lambda r(s,a){y_{k + 1}}\big) - \lambda y_{k+1}^2 \Big)$ \;
}
\caption{\label{alg:mvpi}Mean-Variance Policy Iteration (MVPI)}
\end{algorithm}
At the $k$-th iteration,
we first fix $\pi_k$ and update $y_{k+1}$ (Step 1).
As $J_\lambda(\pi_k, y)$  is quadratic in $y$,
$y_{k+1}$ can be computed analytically as
$y_{k+1} = \sum_{s, a}d_{\pi_k}(s, a)r(s, a) = (1 - \gamma) J(\pi_k)$,
i.e., all we need in this step is $J(\pi_k)$,
which is exactly the performance metric of the policy $\pi_k$.
We, therefore,
refer to Step 1 as \emph{policy evaluation}.
We then fix $y_{k+1}$ and update $\pi_{k+1}$ (Step 2).
Remarkably, Step 2 can be reduced to the following problem:
\begin{align}
\textstyle{ {\pi _{k + 1}} = \arg {\max _\pi } \sum_{s, a} d_\pi(s, a) \hat{r}(s, a; y_{k+1})},
\end{align}
where $\hat{r}(s, a; y) \doteq r(s, a) - \lambda r(s, a)^2 + 2\lambda r(s, a) y$.
In other words, 
to compute $\pi_{k+1}$, 
we need to solve a new MDP, 
which is the same as the original MDP except that the reward function is $\hat{r}$ instead of $r$.
This new reward function $\hat{r}$ does not depend on the policy $\pi$,
avoiding the policy-dependent-reward issue of \citet{papini:risk2019}.
In this step, a new policy $\pi_{k+1}$ is computed.
An intuitive conjecture is that this step is a \emph{policy improvement} step,
and we confirm this with the following proposition:
\begin{proposition}
\label{lem:mvpi}
({Monotonic Policy Improvement})\\ $\forall k, J_\lambda(\pi_{k+1}) \geq J_\lambda(\pi_k)$.
\end{proposition}
Though the monotonic improvement w.r.t.\ the objective $J_\lambda(\pi, y)$ in Eq~\eqref{pb:per-step-reward} follows directly from standard BCCA theories,
Theorem~\ref{lem:mvpi} provides the monotonic improvement w.r.t.\ the objective $J_\lambda(\pi)$ in Eq~\eqref{eq:obj}.
The proof is provided in the appendix.
Given Theorem \ref{lem:mvpi}, 
we can now consider the whole BCCA framework in Algorithm~\ref{alg:mvpi} as a policy iteration framework,
which we call 
\emph{mean-variance policy iteration} (MVPI).
Let $\{\pi_\theta: \theta \in \Theta\}$ be the function class for policy optimization, 
we have
\begin{assumption}
\label{assu:compact}
$\{\theta \in \Theta, y \in \R \mid J_\lambda(\theta, y) \geq J_\lambda(\theta_0)\}$ is compact,
where $\theta_0$ is the initial parameters.
\end{assumption}
\begin{assumption}
\label{assu:params}
$\sup_{\theta \in \Theta} ||\frac{\partial \log \pi_\theta(a|s)}{\partial \theta_i \partial \theta_j} || < \infty$, 
$\sup_{\theta \in \Theta} || \nabla_\theta \log \pi_\theta(a|s) || < \infty$, 
$\Theta$ is open and bounded.
\end{assumption}
\begin{proposition}
\label{thm:mvpi}
(Convergence of MVPI with function approximation)
Under Assumptions \ref{assu:compact} \& \ref{assu:params}, let
\begin{align}
y_{k+1} &\doteq \arg\max_{y} J_\lambda({\theta_k}, y), \\
\theta_{k+1} &\doteq \arg\max_{\theta \in \Theta} J_\lambda (\theta, y_{k+1}), \quad k = 0, 1, \dots
\end{align}
then $J_\lambda(\theta_{k+1}) \geq J_\lambda(\theta_{k})$,
$\{J_\lambda(\theta_k)\}_{k = 1, \dots}$ converges,
and $\lim \inf_{k} ||\nabla_\theta J_\lambda(\theta_k)|| = 0$.
\end{proposition}
\begin{remark}
Assumption~\ref{assu:compact} is standard in BCCA literature (e.g., Theorem 4.1 in \citet{tseng2001convergence}).
Assumption~\ref{assu:params} is standard in policy optimization literature (e.g., Assumption 4.1 in \citet{papini2018stochastic}). 
Convergence in the form of $\lim\inf$ also appears in other literature (e.g., \citet{luenberger1984linear,tseng2001convergence,konda2002thesis,zhang2019provably}).
\end{remark}
The proof is provided in the appendix.
MVPI enjoys great flexibility in that any policy evaluation method and risk-neutral control method can be dropped in off the shelf,
which makes it possible to leverage all the advances in risk-neutral RL.
MVPI differs from the standard policy iteration (PI, e.g., see \citet{bertsekas1996neuro,puterman2014markov,sutton2018reinforcement}) in two key ways: 
\textbf{(1)}  policy evaluation in MVPI requires only a scalar performance metric,  
while standard policy evaluation involves computing the value of all states.
\textbf{(2)}  policy improvement in MVPI considers an augmented reward $\hat{r}$,
which is different at each iteration,
while standard policy improvement always considers the original reward.
Standard PI can be used to solve the policy improvement step in MVPI.

\subsection{Average Reward Setting}
So far we have considered the total reward as the primary performance metric for mean-variance RL.
We now show that MVPI can also be used when we consider the average reward as the primary performance metric.
Assuming the chain induced by $\pi$ is ergodic and letting $\bar{d}_\pi(s)$ be its stationary distribution, 
\citet{filar1989variance,prashanth2013actor} consider the \emph{long-run variance} risk measure
$\Lambda(\pi) \doteq \sum_{s, a}\bar{d}_\pi(s, a) \big( r(s,a) - \bar{J}(\pi) \big)^2$ for the average reward setting,
where $\bar{d}_\pi(s, a) \doteq \bar{d}_\pi(s) \pi(a | s)$ and $\bar{J}(\pi) = \sum_{s,a} \bar{d}_\pi(s, a) r(s, a)$ is the average reward.
We now define a risk-aware objective 
\begin{align}
\label{eq:mvpi-avg}
 \bar{J}_\lambda(\pi) &\doteq \bar{J}(\pi) - \lambda \Lambda(\pi) \\
&= \max_y \sum_{s, a} \bar{d}_\pi(s, a)\hat{r}(s, a; y) - \lambda y^2,
\end{align}
where we have used the Fenchel duality and BCCA can take over to derive MVPI for the average reward setting as Algorithm~\ref{alg:mvpi}.
It is not a coincidence that the only difference between \eqref{pb:per-step-reward} and \eqref{eq:mvpi-avg} is the difference between $d_\pi$ and $\bar{d}_\pi$.
The root cause is that the total discounted reward of an MDP is always equivalent to the average reward of an artificial MDP (up to a constant multiplier), whose transition kernel is $\tilde{p}(s^\prime | s, a) = \gamma p(s^\prime | s, a) + (1 - \gamma) \mu_0(s^\prime)$ (e.g., see Section 2.4 in \citet{konda2002thesis} for details). 

\subsection{Off-Policy Learning}
Off-policy learning has played a key role in improving  data efficiency \citep{lin1992self,mnih2015human} and exploration \citep{osband2016deep,osband2018randomized} in risk-neutral control algorithms.
Previous mean-variance RL methods, however, consider only the on-policy setting and cannot be easily made off-policy.
For example, it is not clear whether perturbation methods for estimating gradients \citep{prashanth2013actor} can be used off-policy.
To reweight terms like $G_0^2\sum_{t=0}^{T-1} \nabla_\theta \log \pi(A_t | S_t)$ from \citet{tamar2012policy,liu2018block} in the off-policy setting,
we would need to compute the product of importance sampling ratios $\Pi_{i=0}^{T-1}\frac{\pi(a_i|s_i)}{\mu(a_i|s_i)}$,
where $\mu$ is the behavior policy.
This product usually suffers from high variance \citep{precup2001off,liu2018breaking} and requires knowing the behavior policy $\mu$,
both of which are practical obstacles in real applications.
By contrast, 
as MVPI works on an augmented MDP directly,
any risk-neutral off-policy learning technique can be used for risk-averse off-policy control directly.
In this paper, we consider MVPI in both on-line and off-line off-policy settings.

\textbf{On-line setting}. In the on-line off-policy setting, 
an agent interacts with the environment following a behavior policy $\mu$ to collect transitions,
which are stored into a replay buffer \citep{lin1992self} for future reuse.
Mujoco robot simulation tasks \citep{brockman2016openai} are common benchmarks for this paradigm \citep{lillicrap2015continuous,haarnoja2018soft},
and TD3 is a leading algorithm in Mujoco tasks.
TD3 is a risk-neutral control algorithm,
reducing the over-estimation bias \citep{hasselt2010double} of DDPG \citep{lillicrap2015continuous},
which is a neural network implementation of 
the deterministic policy gradient theorem \citep{silver2014deterministic}.
Given the empirical success of TD3,
we propose MVPI-TD3 for risk-averse control in this setting.
In the policy evaluation step of MVPI-TD3, 
we set $y_{k+1}$ to the average of the recent $K$ rewards,
where $K$ is a hyperparameter to be tuned and we have assumed the policy changes slowly.
Theoretically, 
we should use a weighted average as $d_\pi(s, a)$ is a discounted distribution.
Though implementing this weighted average is straightforward,
practitioners usually ignore discounting for state visitation in policy gradient methods to improve sample efficiency \citep{mnih2016asynchronous,schulman2015trust,schulman2017proximal,bacon2017option}.
Hence, we do not use the weighted average in MVPI-TD3.
In the policy improvement step of MVPI-TD3,
we sample a mini-batch of transitions from the replay buffer and perform one TD3 gradient update.
The pseudocode of MVPI-TD3 is provided in the appendix.

\textbf{Off-line setting} In the off-line off-policy setting,
we are presented with a batch of transitions $\{s_i, a_i, r_i, s^\prime_i\}_{i=1,\dots,K}$
and want to learn a good target policy $\pi$ for control solely from this batch of transitions.
Sometimes those transitions are generated by following a known behavior policy $\mu$.
But more commonly, 
those transitions are generated from multiple unknown behavior policies,
which we refer to as the behavior-agnostic off-policy setting \citep{nachum2019dualdice}.
Namely, the state-action pairs $(s_i, a_i)$ are distributed according to some unknown distribution $d$,
which may result from multiple unknown behavior policies.
The successor state $s_i^\prime$ is distributed according to $p(\cdot|s_i, a_i)$ and $r_i = r(s, a)$.
The degree of off-policyness in this setting is usually larger than the on-line off-policy setting.

In the off-line off-policy setting,
the policy evaluation step in MVPI becomes the standard \emph{off-policy evaluation} problem (OPE, \citet{ope:thomas2015high,ope:thomas2016,ope:jiang2015doubly,liu2018breaking}),
where we want to estimate a scalar performance metric of a policy with off-line samples.
One promising approach to OPE is \textit{density ratio learning}, 
where we use function approximation to learn the density ratio $\frac{d_\pi(s, a)}{d(s, a)}$ directly, 
which we then use to reweight $r(s, a)$.
All off-policy evaluation algorithms can be integrated into MVPI in a plug-and-play manner.
In the off-line off-policy setting, 
the policy improvement step in MVPI becomes the standard \emph{off-policy policy optimization} problem,
where we can reweight the canonical on-policy actor-critic \citep{sutton2000policy,konda2002thesis} with the density ratio as in \citet{liu2019off} to achieve off-policy policy optimization. 
Algorithm~\ref{alg:off-policy-mvpi} provides an example of Off-line MVPI.

In the on-line off-policy learning setting, 
the behavior policy and the target policy are usually closely correlated (e.g., in MVPI-TD3),
we, therefore, do not need to learn the density ratio.
In the off-line off-policy learning setting,
the dataset may come from behavior policies that are arbitrarily different from the target policy.
We, therefore, resort to density ratios to account for this discrepancy.
Density ratio learning itself is an active research area and is out of scope of this paper.
See \citet{hallak2017consistent,liu2018breaking,gelada2019off,nachum2019dualdice,zhang2020gendice,zhang2020gradientdice,mousavi2020blackbox} for more details about density ratio learning.

\begin{algorithm}
\textbf{Input:} 
A batch of transitions $\{s_i, a_i, r_i, s^\prime_i\}_{i = 1, \dots, K}$, a learning rate $\alpha$, and the number of policy updates $T$ \;
\While{True}{
 Learn the density ratio $\rho_\pi(s, a) \doteq \frac{d_\pi(s, a)}{d(s, a)}$ with $\{s_i, a_i, r_i, s^\prime_i\}_{i = 1, \dots, K}$\; 
\tcp{For example, use GradientDICE \citep{zhang2020gradientdice}}
 $y \gets \frac{1}{K}\sum_{i=1}^K \rho_\pi(s_i, a_i)r_i$ \;
\For{$i = 1, \dots, K$}{
  $\hat{r}_i \gets r_i - \lambda r_i^2 + 2\lambda r_i y$ \;
  $a_i^\prime \sim \pi( \cdot | s_i^\prime)$
}
\For{$j=1, \dots, T$}{
  Learn $q_\pi(s, a)$ with $\{s_i, a_i, \hat{r}_i, s^\prime_i, a_i^\prime\}_{i = 1, \dots, K}$\;
  \tcp{For example, use TD(0) \citep{sutton1988learning} in $\mathcal{S} \times \mathcal{A}$}
  Learn $\rho_\pi(s, a)$ with $\{s_i, a_i, r_i, s^\prime_i\}_{i = 1, \dots, K}$\; 
$\theta \gets \theta + \alpha \rho_\pi(s_i, a_i) \nabla_\theta \log \pi(a_i | s_i) q_\pi(s_i, a_i)$, where $i$ is randomly selected
}
}
\caption{\label{alg:off-policy-mvpi} Off-line MVPI}
\end{algorithm}

\begin{savenotes}
\begin{table*}[t]
\centering
\begin{tabular}{l|llll|llll}
\hline
& $\Delta_\text{J}^{\text{TRVO}}$ & $\Delta_\text{mean}^{\text{TRVO}}$ & $\Delta_\text{variance}^{\text{TRVO}}$ & $\Delta_\text{SR}^{\text{TRVO}}$ & $\Delta_\text{J}^{\text{MVPI}}$ & $\Delta_\text{mean}^{\text{MVPI}}$ & $\Delta_\text{variance}^{\text{MVPI}}$ & $\Delta_\text{SR}^{\text{MVPI}}$ \\ \hline
InvertedP.& -321\% & -1\% & $3 \cdot 10^7$\% & -100\%& 0\% & 0\% & 0\% & 0\%\\ \hline
InvertedD.P.& -367\% & -17\% & 365\% & -62\%& 42\% & 1\% & -41\% & 32\%\\ \hline
HalfCheetah& 92\% & -87\% & -92\% & -54\%& 99\% & -44\% & -98\% & 285\%\\ \hline
Walker2d& -1\% & -61\% & 0\% & -61\%& 88\% & -55\% & -88\% & 30\%\\ \hline
Swimmer& -10\% & -7\% & 18\% & -14\%& 4\% & -2\% & -53\% & 42\%\\ \hline
Hopper& 57\% & -25\% & -57\% & 13\%& 72\% & -3\% & -71\% & 81\%\\ \hline
Reacher& -48\% & -12\% & 102\% & 21\%& -1\% & -1\% & 0\% & -1\%\\ \hline
Ant& 97\% & -85\% & -97\% & -18\%& 96\% & -53\% & -96\% & 130\%\\ \hline
\end{tabular}
\caption{\label{tab:normalized} 
Normalized statistics of TRVO and MVPI-TD3.
MVPI is shorthand for MVPI-TD3 in this table.
For $\text{algo} \in \{\text{MVPI-TD3}, \text{TRVO}, \text{TD3}\}$,
we compute the risk-aware performance metric as 
 $\text{J}_{\text{algo}} \doteq \text{mean}_{\text{algo}} - \lambda \text{variance}_{\text{algo}}$ with $\lambda = 1$,
where $\text{mean}_{\text{algo}}$ and $\text{variance}_{\text{algo}}$ are mean and variance of the 100 evaluation episodic returns.
Then we compute the normalized statistics as 
$\Delta_\text{J}^\text{algo} \doteq \frac{\text{J}_{\text{algo}} - \text{J}_{\text{TD3}}}{|\text{J}_{\text{TD3}}|}$,
$\Delta_\text{mean}^\text{algo} \doteq \frac{\text{mean}_{\text{algo}} - \text{mean}_{\text{TD3}}}{|\text{mean}_{\text{TD3}}|} $,
$\Delta_\text{variance}^\text{algo} \doteq \frac{\text{variance}_{\text{algo}} - \text{variance}_{\text{TD3}}}{|\text{variance}_{\text{TD3}}|}$,
$\text{SR}_\text{algo} \doteq \frac{\text{mean}_\text{algo}}{\sqrt{\text{variance}_\text{algo}}}$,
$\Delta_\text{SR}^\text{algo} \doteq \frac{\text{SR}_{\text{algo}} - \text{SR}_{\text{TD3}}}{|\text{SR}_{\text{TD3}}|}$
Both MVPI-TD3 and TRVO are trained with $\lambda = 1$.
All $J_\text{algo}$ are averaged over 10 independent runs.
}
\end{table*}
\end{savenotes}

\section{Experiments}
All curves in this section are averaged over 10 independent runs with shaded regions indicate standard errors.
All implementations are publicly available.\footnote{\url{https://github.com/ShangtongZhang/DeepRL}} 

\textbf{On-line learning setting.}
In many real-world robot applications, e.g., in a warehouse, 
it is crucial that the robots' performance be consistent. 
In such cases, risk-averse RL is an appealing option to train robots.
Motivated by this,
we benchmark MVPI-TD3 on eight Mujoco robot manipulation tasks from OpenAI gym.
Though Mujoco tasks have a stochastic initial state distribution,
they are usually equipped with a deterministic transition kernel.
To make them more suitable for investigating risk-averse control algorithms,
we add a Gaussian noise $\mathcal{N}(0, 0.1^2)$ to every action.
As we are not aware of any other off-policy mean-variance RL method, 
we use several recent on-policy mean-variance RL method as baselines,
namely, the methods of \citet{tamar2012policy,prashanth2013actor},
MVP \citep{liu2018block}, and TRVO \citep{papini:risk2019}.
The methods of \citet{tamar2012policy,prashanth2013actor} and MVP are not designed for deep RL settings.
To make the comparison fair,
we improve those baselines with parallelized actors to stabilize the training of neural networks as in \citet{mnih2016asynchronous}.\footnote{They are on-policy algorithms so we cannot use experience replay.}
TRVO is essentially MVPI with TRPO for the policy improvement.
We, therefore, implement TRVO as MVPI with Proximal Policy Optimization (PPO, \citealt{schulman2017proximal}) to improve its performance. 
We also use the vanilla risk-neutral TD3 as a baseline.
We use two-hidden-layer neural networks for function approximation. 

We run each algorithm for $10^6$ steps and evaluate the algorithm every $10^4$ steps for 20 episodes.
We report the mean of those 20 episodic returns against the training steps in Figure~\ref{fig:mvpi-td3}.
The curves are generated by setting $\lambda = 1$.
More details are provided in the appendix.
The results show that MVPI-TD3 outperforms all risk-averse baselines in all tested domains (in terms of both final episodic return and learning speed).
Moreover, the curves of the methods from the total reward perspective are always flat in all domains
with only one exception that MVP achieves a reasonable performance in \texttt{Reacher},
though exhaustive hyperparameter tuning is conducted,
including $\lambda$ and $\xi$.
This means
they fail to achieve the risk-performance trade-off in our tested domains,
indicating that those methods are not able to scale up to Mujoco domains with neural network function approximation.
Those flat curves suggest that perturbation-based gradient estimation in \citet{prashanth2013actor} may not work well with neural networks, 
and the $G_0^2 \sum_{t=0}^{T-1} \nabla_\theta \log \pi(a_t | s_t)$ term in \citet{tamar2012policy} and MVP
may suffer from high variance, yielding instability.
By contrast, the two algorithms from the per-step reward perspective (MVPI-TD3 and TRVO) do learn a reasonable policy,
which experimentally supports 
our argument that 
optimizing the variance of the per-step reward is more efficient and scalable than optimizing the variance of the total reward.

\begin{figure*}[t]
\centering
\includegraphics[width=0.8\textwidth]{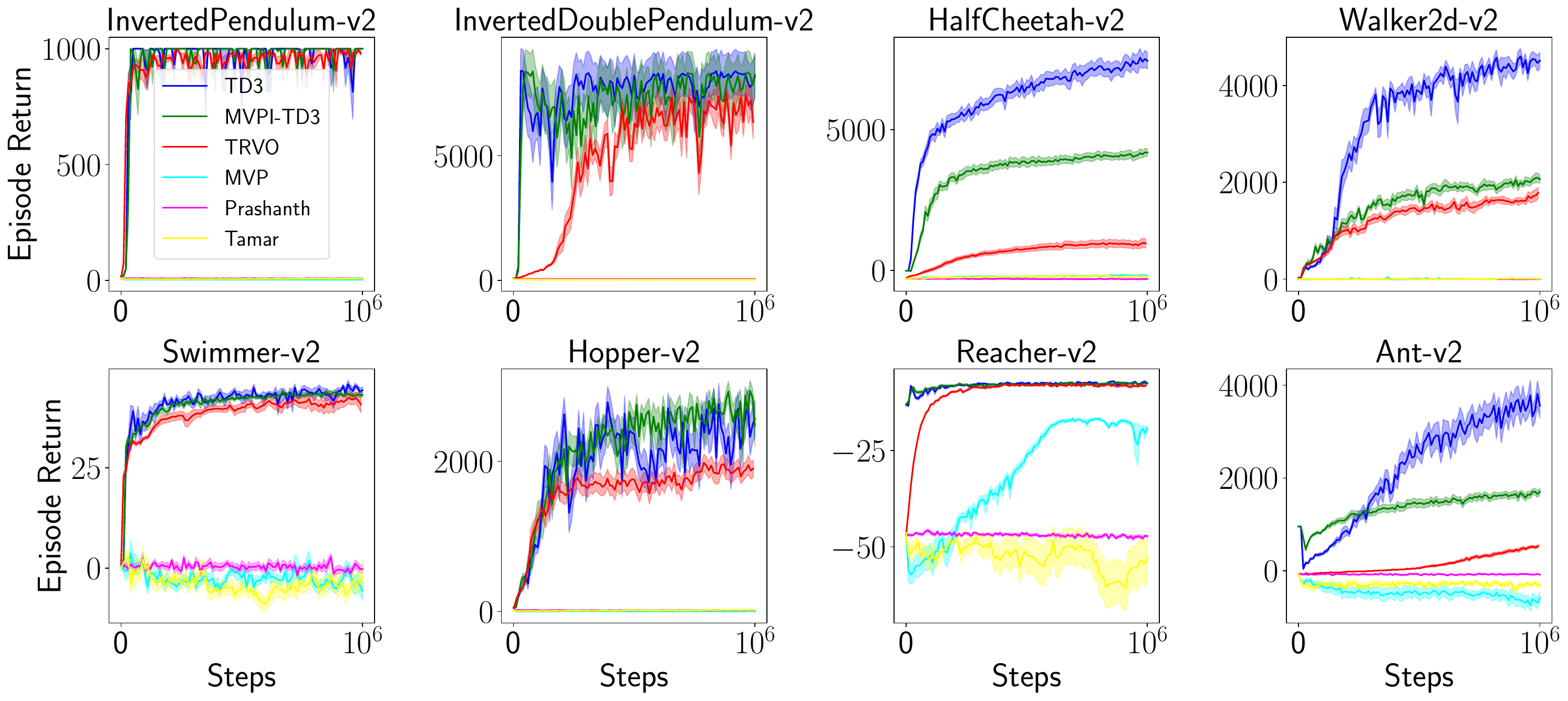}
\caption{\label{fig:mvpi-td3} Training progress of MVPI-TD3 and baseline algorithms. 
Curves are averaged over 10 independent runs with shaded regions indicating standard errors.}
\end{figure*}

As shown in Figure~\ref{fig:mvpi-td3},
the vanilla risk-neutral TD3 outperforms all risk-averse algorithms (in terms of episodic return).
This is expected as 
it is in general hard for a risk-averse algorithm to outperform its risk-neutral counterpart in terms of a risk-neutral performance metric.
We now compare TD3, MVPI-TD3 and TRVO in terms of a risk-aware performance metric.
To this end, we test the agent at the end of training for an extra 100 episodes
to compute a risk-aware performance metric.
We report the normalized statistics in Table~\ref{tab:normalized}. 
The results show that MVPI-TD3 outperforms TD3 in 6 out of 8 tasks in terms of the risk-aware performance metric.
Moreover, MVPI-TD3 outperforms TRVO in 7 out of 8 tasks.
We also compare the algorithms in terms of the sharp ratio (SR, \citealt{sharpe1966mutual}).
Although none of the algorithms optimizes SR directly, 
MVPI-TD3 outperforms TD3 and TRVO in 6 and 7 tasks respectively in terms of SR.
This performance boost of MVPI-TD3 over TRVO indeed results from the performance boost of TD3 over PPO,
and it is the flexibility of MVPI that makes this off-the-shelf application TD3 in risk-averse RL possible. 
We also provide versions of Figure~\ref{fig:mvpi-td3} and Table~\ref{tab:normalized} with $\lambda=0.5$ and $\lambda=2$ in the appendix.
The relative performance is the same as $\lambda = 1$.
Though we propose to optimize the variance of the per-step reward as a proxy for optimizing the variance of the total reward,
we also provide in the appendix curves comparing MVPI-TD3 and TRVO w.r.t. risk measures consisting of both mean and variance of the per-step reward.
With those risk measures, MVPI-TD3 generally outperforms TRVO as well.

\textbf{Off-line learning setting.} 
We consider an infinite horizon MDP (Figure~\ref{fig:off-policy-mvpi-MDP}).
Two actions $a_0$ and $a_1$ are available at $s_0$,
and we have $p(s_3|s_0, a_1) = 1, p(s_1|s_0, a_0) = p(s_2|s_0, a_0) = 0.5$.
The discount factor is $\gamma = 0.7$ and the agent is initialized at $s_0$.
We consider the objective $J_\lambda(\pi)$ in Eq~\eqref{eq:obj}. 
If $\lambda = 0$,
the optimal policy is to choose $a_0$.
If $\lambda$ is large enough,
the optimal policy is to choose $a_1$.
We consider the behavior-agnostic off-policy setting, 
where the sampling distribution $d$ satisfies $d(s_0, a_0) = d(s_0, a_1) = d(s_1) = d(s_2) = d(s_3) = 0.2$.
This sampling distribution may result from multiple unknown behavior policies.
Although the representation is tabular,
we use a softmax policy.
So the problem we consider is \emph{nonlinear} and \emph{nonconvex}.
As we are not aware of any other behavior-agnostic off-policy risk-averse RL method, 
we benchmark only Off-line MVPI (Algorithm~\ref{alg:off-policy-mvpi}).
Details are provided in the appendix. 
We report the probability of selecting $a_0$ against training iterations.
As shown in Figure~\ref{fig:off-policy-mvpi}, 
$\pi(a_0 | s_0)$ decreases as $\lambda$ increases,
indicating Off-line MVPI copes well with different risk levels. 
The main challenge in Off-line MVPI rests on learning the density ratio.
Scaling up density ratio learning algorithms reliably to more challenging domains like Mujoco is out of the scope of this paper.

\begin{figure}[h]
\centering
\includegraphics[width=0.25\textwidth]{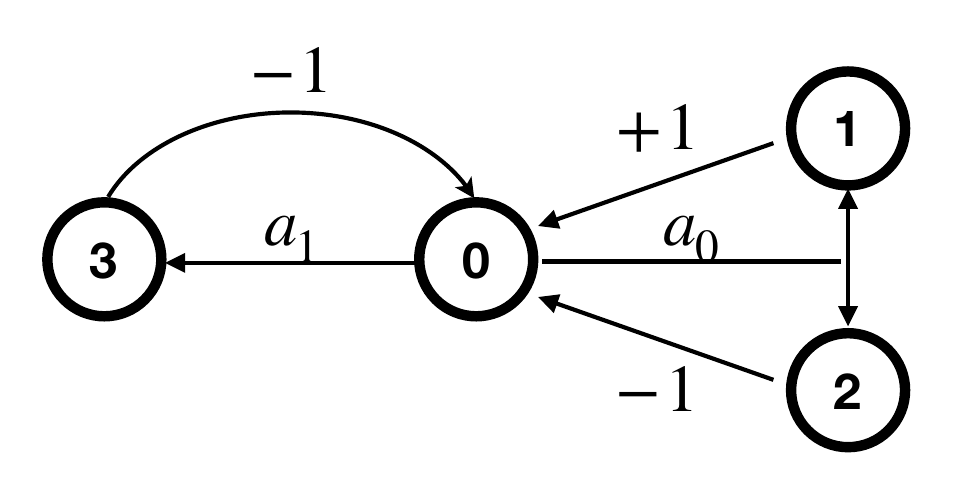}
\caption{\label{fig:off-policy-mvpi-MDP} A tabular MDP}
\end{figure}

\begin{figure}
\centering
\includegraphics[width=0.35\textwidth]{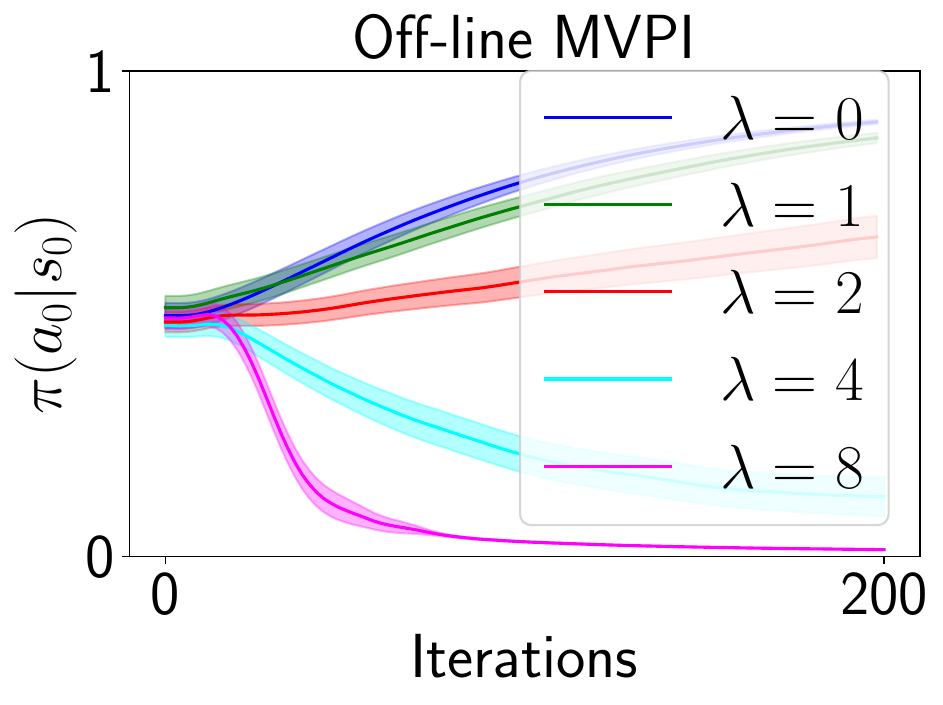}
\caption{\label{fig:off-policy-mvpi} The training progress of Off-line MVPI. Curves are averaged over 30 independent runs with shaded regions indicating standard errors.}
\end{figure}

\section{Related Work}
Both MVPI and \citet{papini:risk2019} consider the per-step reward perspective for mean-variance RL.
In this work, 
we mainly use the variance of the per-step reward as a proxy (upper bound) for optimizing the variance of the total reward.
Though TRVO in \citet{papini:risk2019} is the same as instantiating MVPI with TRPO,
the derivation is dramatically different.
In particular,
it is not clear whether the performance-lower-bound-based derivation for TRVO can be adopted to deterministic policies, off-policy learning, or other policy optimization paradigms,
and this is not explored in \citet{papini:risk2019}.
By contrast, MVPI is compatible with \emph{any} existing risk-neural policy optimization technique.
Furthermore, MVPI works for both the total discounted reward setting and the average reward setting.
It is not clear how the performance lower bound in \citet{papini:risk2019}, 
which plays a central role in TRVO, 
can be adapted to the average reward setting.
All the advantages of MVPI over TRVO result from addressing the policy-dependent-reward issue in \citet{papini:risk2019}.

While the application of Fenchel duality in RL is not new, previously it has been used only to address double sampling issues (e.g., \citet{liu2015finite,dai2017sbeed,liu2018block,nachum2019dualdice}).
By contrast,
we use Fenchel duality together with BCAA to address the policy-dependent-reward issue in \citet{papini:risk2019} and derive a policy iteration framework that appears to be novel to the RL community.


Besides variance, 
value at risk (VaR, \citealt{chow2018risk}), 
conditional value at risk (CVaR, \citealt{chow2014algorithms,tamar2015optimizing,chow2018risk}),
sharp ratio \citep{tamar2012policy},
and exponential utility \citep{howard1972risk,borkar2002q}
are also used for risk-averse RL.
In particular, it is straightforward to consider exponential utility for the per-step reward,
which, however, suffers from the same problems as the exponential utility for the total reward, e.g., it overflows easily \citep{gosavi2014beyond}.


\section{Conclusion}
In this paper, we propose MVPI for risk-averse RL.
MVPI enjoys great flexibility such that any policy evaluation method and risk-neutral control method
can be dropped in for risk-averse control off the shelf, 
in both on- and off-policy settings.
This flexibility dramatically reduces the gap between risk-neutral control and risk-averse control.
To the best of our knowledge, 
MVPI is the first empirical success of risk-averse RL in Mujoco robot simulation domains, and is also the first success of off-policy risk-averse RL and risk-averse RL with deterministic polices.
Deterministic policies play an important role in reducing the variance of a policy \citep{silver2014deterministic}. 
Off-policy learning is important for improving data efficiency \citep{mnih2015human} and exploration \citep{osband2018randomized}.
Incorporating those two elements in risk-averse RL appears novel and is key to the observed performance improvement.

Possibilities for future work include
considering other risk measures (e.g., VaR and CVaR) of the per-step reward random variable,
integrating more advanced off-policy policy optimization techniques (e.g., \citealt{nachum2019algaedice}) in off-policy MVPI,
optimizing $\lambda$ with meta-gradients \citep{xu2018meta}, 
analyzing the sample complexity of MVPI,
and developing theory for approximate MVPI. 


{\small
\bibliography{ref}
}


\newpage
\onecolumn
\appendix

\section{Proofs}
\subsection{Proof of Proposition~\ref{lem:mvpi}}
\begin{proof}
\begin{align}
&J_\lambda(\pi_{k + 1}) \\
=& \sum_{s, a}d_{\pi_{k+1}}(s, a)(r(s, a) - \lambda r(s, a)^2) + \lambda \max_y (2\sum_{s, a} d_{\pi_{k+1}}(s, a)r(s,a)y - y^2) \\
\geq& \sum_{s, a}d_{\pi_{k+1}}(s, a)(r(s, a) - \lambda r(s, a)^2) + \lambda (2\sum_{s, a} d_{\pi_{k+1}}(s, a)r(s,a)y_{k+1} - y_{k+1}^2) \\
=& \sum_{s, a}d_{\pi_{k+1}}(s, a) \big(r(s,a) - \lambda r{(s,a)^2} + 2\lambda r(s,a){y_{k + 1}}\big) - \lambda y_{k + 1}^2 \\
\geq& \sum_{s, a}d_{\pi_k}(s, a)\big(r(s,a) - \lambda r{(s,a)^2} + 2\lambda r(s,a){y_{k + 1}}\big) - \lambda y_{k+1}^2 \intertext{\hfill (By definition, $\pi_{k+1}$ is the maximizer.)}\\
=& \sum_{s, a}d_{\pi_{k}}(s, a)(r(s, a) - \lambda r(s, a)^2) + \lambda (2\sum_{s, a} d_{\pi_k}(s, a)r(s,a)y_{k + 1} - y_{k + 1}^2) \\
=& \sum_{s, a}d_{\pi_{k}}(s, a)(r(s, a) - \lambda r(s, a)^2) + \lambda \max_y (2\sum_{s, a} d_{\pi_k}(s, a)r(s,a)y - y^2) \intertext{\hfill (By definition, $y_{k+1}$ is the maximizer of the quadratic.)} \\
=&J_\lambda(\pi_k)
\end{align}
\end{proof}

\subsection{Proof of Proposition~\ref{thm:mvpi}}
\begin{lemma}
\label{lem:lipschitz}
Under Assumption~\ref{assu:params},
$\nabla_\theta J_\lambda(\theta)$ is Lipschitz continuous in $\theta$.
\end{lemma}
\begin{proof}
By definition,
\begin{align}
\nabla J_\lambda(\theta) = \nabla \E[R] - \nabla \lambda \E[R^2] + 2 \lambda \E[R] \nabla \E[R].
\end{align}
The policy gradient theorem \citep{sutton2000policy} and the boundedness of $\nabla \log \pi_\theta(a|s)$ imply that $\nabla \E[R]$ is bounded.
So $\E[R]$ is Lipschitz continuous.
Lemma B.2 in \citet{papini2018stochastic} shows that 
the Hessian of $\E[R]$ is bounded.
So $\nabla \E[R]$ is Lipschitz continuous.
So does $\nabla \E[R^2]$.
Together with the boundedness of $\E[R]$,
it is easy to see $\nabla J_\lambda(\theta)$ is Lipschitz continuous.
\end{proof}

We now prove Proposition~\ref{thm:mvpi}.
\begin{proof}
Under Assumption~\ref{assu:compact}, Theorem 4.1(c) in \citet{tseng2001convergence}
shows
that the limit of any convergent subsequence $\{(\theta_k, y_k)\}_{k \in \mathcal{K}}$,
referred to as $(\theta_\mathcal{K}, y_\mathcal{K})$,
satisfies $\nabla_\theta J_\lambda(\theta_\mathcal{K}, y_\mathcal{K}) = 0$ and $\nabla_y J_\lambda(\theta_\mathcal{K}, y_\mathcal{K}) = 0$.
In particular, that Theorem 4.1(c) is developed for general block coordinate ascent algorithms with $M$ blocks.
Our MVPI is a special case with two blocks (i.e., $\theta$ and $y$).
With only two blocks, the conclusion of Theorem 4.1(c) follows immediately from Eq (7) and Eq (8) in \citet{tseng2001convergence},
without involving the assumption that the maximizers of the $M - 2$ blocks are unique.

As $J_\lambda(\theta, y)$ is quadratic in $y$,
$\nabla_y J_\lambda(\theta_\mathcal{K}, y_\mathcal{K}) = 0$ implies 
$y_{\mathcal{K}} = \arg\max_y J_\lambda(\theta_{\mathcal{K}}, y) = (1 - \gamma) J(\theta_{\mathcal{K}})$.
Recall the Fenchel duality
$x^2 = \max_z f(x, z)$,
where $f(x, z) \doteq 2xz - z^2$.
Applying Danskin's theorem (Proposition B.25 in \citet{bertsekas1995nonlinear}) to Fenchel duality 
 yields
\begin{align}
\label{eq:danskin}
\frac{\partial x^2}{\partial x} = \frac{\partial f(x, \arg\max_z f(x, z))}{\partial x}.
\end{align}
Note Danskin's theorem shows that we can treat $\arg\max_z f(x, z)$ as a constant independent of $x$ when computing the gradients in the RHS of Eq~\eqref{eq:danskin}.
Applying Danskin's theorem in the Fenchel duality used in Eq~\eqref{eq:obj} yields
\begin{align}
\label{eq:stationary_point}
\nabla_\theta J_\lambda(\theta_{\mathcal{K}}) = \nabla_\theta J_\lambda(\theta_\mathcal{K}, y_\mathcal{K}) = 0.
\end{align}
Eq~\eqref{eq:stationary_point} can also be easily verified without invoking Danskin's theorem by expanding the gradients explicitly. 
Eq~\eqref{eq:stationary_point} indicates that the subsequence $\{\theta_k\}_{k \in \mathcal{K}}$ converges to a stationary point of $J_\lambda(\theta)$.

Theorem~\ref{lem:mvpi} establishes the monotonic policy improvement when we search over all possible policies (The $\arg\max$ of Step 2 in Algorithm~\ref{alg:mvpi} is taken over all possible policies).
Fortunately, 
the proof of Theorem~\ref{lem:mvpi} can also be used (up to a change of notation) to establish that 
\begin{align}
\label{eq:improv}
J_\lambda(\theta_{k+1}) \geq J_\lambda(\theta_k).
\end{align}
In other words, the monotonic policy improvement also holds  when we search over $\Theta$.
Eq~\eqref{eq:improv} and the fact that $J_\lambda(\theta)$ is bounded from above imply that $\{J_\lambda(\theta_k)\}_{k = 1, \dots}$ converges to some $J_*$.

Let $\Theta_0 \doteq \{\theta \in \Theta \mid J_\lambda(\theta) \geq J_\lambda(\theta_0) \}$.
We first show $\Theta_0$ is compact.
Let $\{\theta^i\}_{i=1, \dots}$ be any convergent sequence in $\Theta_0$ and $\theta^\infty$ be its limit.
We define $y^i \doteq \arg\max_y J_\lambda(\theta^i, y) = (1 - \gamma) J(\theta^i)$ for $i = 1, \dots, \infty$.
The proof of Lemma~\ref{lem:lipschitz} shows $J(\theta)$ is Lipschitz continuous in $\theta$,
indicating $\{\theta^i, y^i\}$ converges to $\{\theta^\infty, y^\infty\}$.
As $J_\lambda(\theta^i, y^i) = J_\lambda(\theta^i) \geq J_\lambda(\theta_0)$,
Assumption~\ref{assu:compact} implies $J_\lambda(\theta^\infty, y^\infty) \geq J_\lambda(\theta_0)$,
i.e., $J_\lambda(\theta^\infty) \geq J_\lambda(\theta_0)$, $\theta^\infty \in \Theta_0$.
So $\Theta_0$ is compact.
As $\{\theta_k\}$ is contained in $\Theta_0$, 
there must exist a convergent subsequence,
indicating 
\begin{align}
\lim \inf_k || \nabla_\theta J_\lambda(\theta_k) || = 0.
\end{align}
\end{proof}

\section{Experiment Details}

The pseudocode of MVPI-TD3 and our TRVO (MVPI-PPO) are provide in Algorithms \ref{alg:mvpi-td3} and \ref{alg:mvppo} respectively.

\begin{algorithm}
\textbf{Input:} \;
$\theta, \psi$: parameters for the deterministic policy $\pi$ and the value function $q_\pi$ \;
$K$: number of recent rewards for estimating the policy performance \;
$\lambda$: weight of the variance penalty \;\;
Initialize the replay buffer $\mathcal{M}$ \;
Initialize $S_0$ \;
\For{$t = 0, \dots, $}{
$A_t \gets \pi(S_t) + \mathcal{N}(0, \sigma^2)$ \;
Execute $A_t$, get $R_{t+1}, S_{t+1}$ \;
Store $(S_t, A_t, R_{t+1}, S_{t+1})$ into $\mathcal{M}$ \;
$y \gets \frac{1}{K} \sum_{i = t - K + 2}^{t + 1} R_t$ \;
Sample a mini-batch $\{s_i, a_i, r_i, s_i^\prime \}_{i = 1, \dots, N}$ from $\mathcal{M}$ \;
\For{$i = 1, \dots, N$}{
  $\hat{r}_i \gets r_i - \lambda r_i^2 + 2\lambda r_i y$ 
}
Use TD3 with $\{s_i, a_i, \hat{r}_i, s_i^\prime \}_{i = 1, \dots, N}$ to optimize $\theta$ and $\psi$ \;
$t \gets t + 1$
}
\caption{\label{alg:mvpi-td3} MVPI-TD3}
\end{algorithm}

\begin{algorithm}
\textbf{Input:} \;
$\theta, \psi$: parameters for the policy $\pi$ and the value function $v_\pi$ \;
$K, \lambda$: rollout length and weight for variance \;\;
\While{True}{
 Empty a buffer $\mathcal{M}$ \;
 Run $\pi$ for $K$ steps in the environment, storing $\{s_i, a_i, r_i, s_{i+1}\}_{i = 1, \dots, K}$ into $\mathcal{M}$ \;
 $y \gets \frac{1}{K}\sum_{i=1}^K r_i$ \;
\For{$i = 1, \dots, K$}{
  $\hat{r}_i \gets r_i - \lambda r_i^2 + 2\lambda r_i y$ \;
}
Use PPO with $\{s_i, a_i, \hat{r}_i, s_{i+1}\}_{i = 1, \dots, K}$ to optimize $\theta$ and $\psi$
}
\caption{\label{alg:mvppo} MVPI-PPO}
\end{algorithm}

\textbf{Task Selection:} 
We use eight Mujoco tasks from Open AI gym~\footnote{\url{https://gym.openai.com/}}\citep{brockman2016openai} and implement the tabular MDP in Figure~\ref{fig:off-policy-mvpi}a by ourselves.

\textbf{Function Parameterization:}
For MVPI-TD3 and TD3, 
we use the same network architecture as \citet{fujimoto2018addressing}.
For TRVO (MVPI-PPO), the methods of \citet{tamar2012policy,prashanth2013actor}, and MVP,
we use the same network architecture as \citet{schulman2017proximal}.

\textbf{Hyperparameter Tuning:}
For MVPI-TD3 and TD3, 
we use the same hyperparameters as \citet{fujimoto2018addressing}.
In particular, for MVPI-TD3, we set $K = 10^4$.
For TRVO (MVPI-PPO),
we use the same hyperparameters as \citet{schulman2017proximal}.
We implement the methods of \citet{prashanth2013actor,tamar2012policy} and MVP with multiple parallelized actors like A2C in \citet{baselines} and inherit the common hyperparameters from \citet{baselines}.

\textbf{Hyperparameters of \citet{prashanth2013actor}:}
To increase stability, 
we treat $\lambda$ as a hyperparameter instead of a variable.
Consequently, $\xi$ does not matter.
We tune $\lambda$ from $\{0.5, 1, 2\}$.
We set the perturbation $\beta$ in \citet{prashanth2013actor} to $10^{-4}$.
We use 16 parallelized actors.
The initial learning rate of the RMSprop optimizer is $7 \times 10^{-5}$,
tuned from $\{7 \times 10^{-5}, 7 \times 10^{-4}, 7 \times 10^{-3}\}$.
We also test the Adam optimizer, 
which performs the same as the RMSprop optimizer.
We use policy entropy as a regularization term,
whose weight is 0.01.
The discount factor is 0.99.
We clip the gradient by norm with a threshold 0.5.

\textbf{Hyperparameters of \citet{tamar2012policy}:}
We tune $\lambda$ from $\{0.5, 1, 2\}$.
We use $\xi = 50$, tuned from $\{1, 10, 50, 100\}$.
We set the initial learning rate of the RMSprop optimizer to $7 \times 10^{-4}$,
tuned from $\{7 \times 10^{-5}, 7 \times 10^{-4}, 7 \times 10^{-3}\}$.
We also test the Adam optimizer, 
which performs the same as the RMSprop optimizer.
The learning rates for the running estimates of $\E[G_0]$ and $\V(G_0)$ is 100 times of the initial learning rate of the RMSprop optimizer.
We use 16 parallelized actors.
We use policy entropy as a regularization term,
whose weight is 0.01.
We clip the gradient by norm with a threshold 0.5.

\textbf{Hyperparameters of \citet{liu2018block}:}
We tune $\lambda$ from $\{0.5, 1, 2\}$.
We set the initial learning rate of the RMSprop optimizer to $7 \times 10^{-4}$,
tuned from $\{7 \times 10^{-5}, 7 \times 10^{-4}, 7 \times 10^{-3}\}$.
We also test the Adam optimizer, 
which performs the same as the RMSprop optimizer.
We use 16 parallelized actors.
We use policy entropy as a regularization term,
whose weight is 0.01.
We clip the gradient by norm with a threshold 0.5.

\textbf{Computing Infrastructure:} 
We conduct our experiments on an Nvidia DGX-1 with PyTorch,
though no GPU is used.

In our off-line off-policy experiments, 
we set $K$ to $10^3$ and use tabular representation for $\rho_\pi, q_\pi$.
For $\pi$, we use a softmax policy with tabular logits.

\section{Other Experimental Results}
We report the empirical results with $\lambda=0.5$ and $\lambda=2$ in Figure \ref{fig:mvpi-td3-0.5}, Table~\ref{tab:normalized-0.5}, Figure~\ref{fig:mvpi-td3-2}, and Table~\ref{tab:normalized-2}.
We report the empirical results comparing MVPI-TD3 and TRVO w.r.t. $J_\text{reward}$ in Figures~\ref{fig:mvpi-td3-per-step-reward-0.5}, \ref{fig:mvpi-td3-per-step-reward-1}, and \ref{fig:mvpi-td3-per-step-reward-2}.
For an algorithm, $J_\text{reward}$ is defined as 
\begin{align}
J_\text{reward} \doteq \text{mean}(\{r_i\}) - \lambda \text{variance}(\{r_i\}),
\end{align}
where $\{r_i\}$ are the rewards during the evaluation episodes.
We do not use the $\gamma$-discounted reward in computing $J_\text{reward}$ because we use the direct average for the policy evaluation step in both MVPI-TD3 and TRVO.
\begin{figure}[h]
\centering
\includegraphics[width=0.8\textwidth]{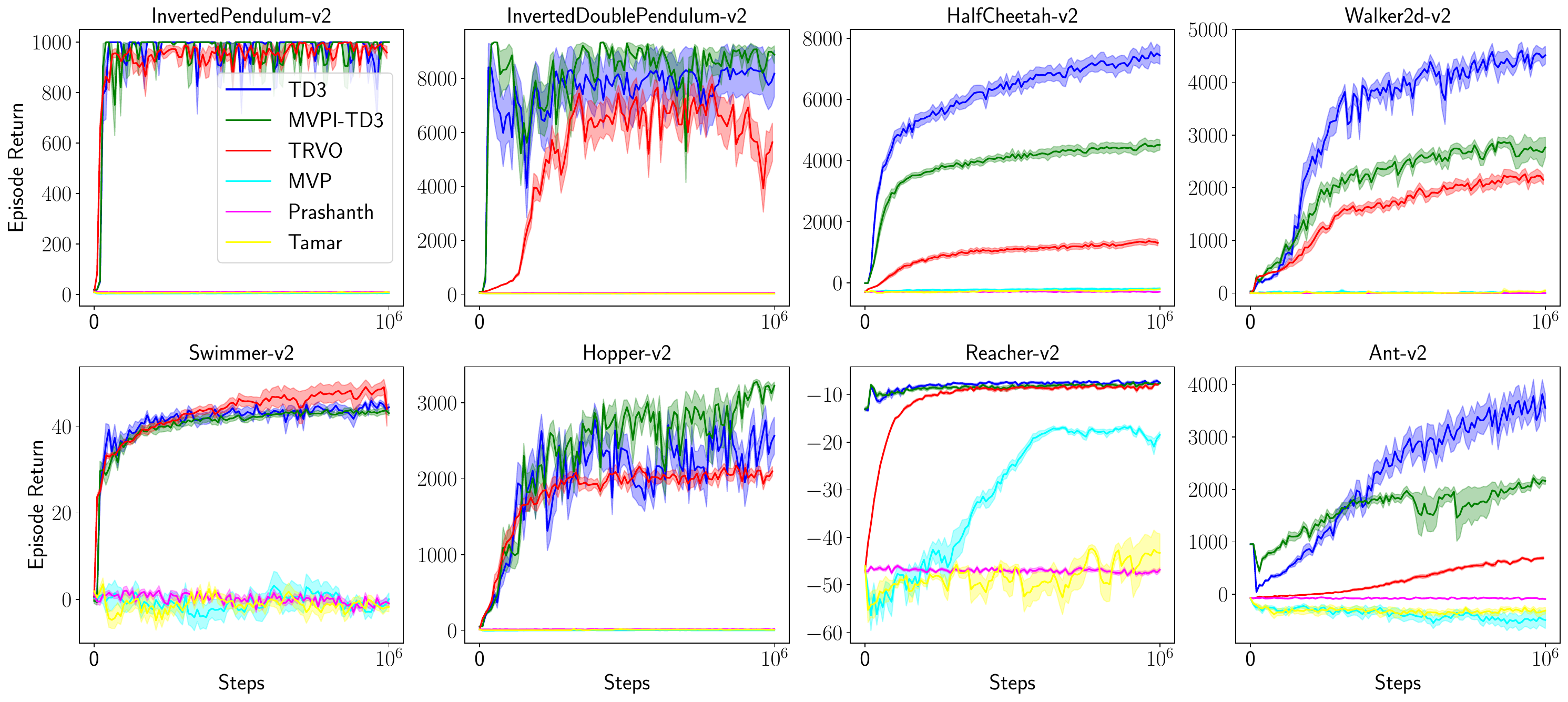}
\caption{\label{fig:mvpi-td3-0.5} Figure~\ref{fig:mvpi-td3} with $\lambda=0.5$.}
\end{figure}

\begin{table}[h]
\centering
\begin{tabular}{l|llll|llll}
\hline
& $\Delta_\text{J}^{\text{TRVO}}$ & $\Delta_\text{mean}^{\text{TRVO}}$ & $\Delta_\text{variance}^{\text{TRVO}}$ & $\Delta_\text{SR}^{\text{TRVO}}$ & $\Delta_\text{J}^{\text{MVPI}}$ & $\Delta_\text{mean}^{\text{MVPI}}$ & $\Delta_\text{variance}^{\text{MVPI}}$ & $\Delta_\text{SR}^{\text{MVPI}}$ \\ \hline
InvertedP.& -571\% & -2\% & $10^8$\% & -100\%& 0\% & 0\% & 0\% & 0\%\\ \hline
InvertedD.P.& -269\% & -31\% & 266\% & -64\%& 4\% & 9\% & -4\% & 11\%\\ \hline
HalfCheetah& 81\% & -82\% & -81\% & -59\%& 97\% & -40\% & -96\% & 188\%\\ \hline
Walker2d& -17\% & -49\% & 15\% & -52\%& 81\% & -37\% & -80\% & 41\%\\ \hline
Swimmer& -2\% & 8\% & 165\% & -34\%& 1\% & -3\% & -69\% & 76\%\\ \hline
Hopper& -9\% & -17\% & 9\% & -20\%& 94\% & 27\% & -93\% & 364\%\\ \hline
Reacher& -34\% & -13\% & 98\% & 20\%& -12\% & -5\% & 35\% & 10\%\\ \hline
Ant& 94\% & -80\% & -94\% & -19\%& 80\% & -41\% & -80\% & 32\%\\ \hline
\end{tabular}
\caption{\label{tab:normalized-0.5} 
Table~\ref{tab:normalized} with $\lambda=0.5$. 
}
\end{table}

\begin{figure}[h]
\centering
\includegraphics[width=0.8\textwidth]{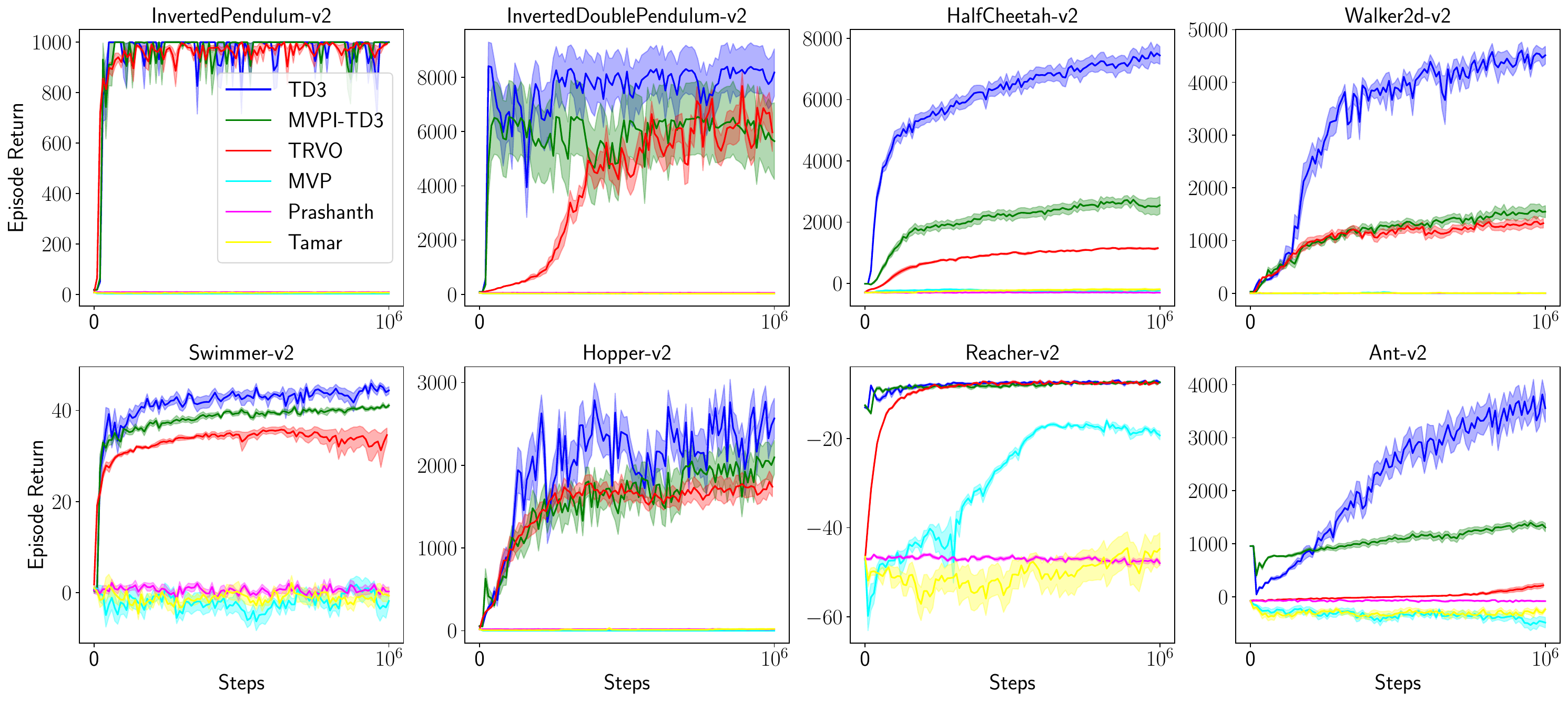}
\caption{\label{fig:mvpi-td3-2} Figure~\ref{fig:mvpi-td3} with $\lambda=2$.}
\end{figure}

\begin{table}[h]
\centering
\begin{tabular}{l|llll|llll}
\hline
& $\Delta_\text{J}^{\text{TRVO}}$ & $\Delta_\text{mean}^{\text{TRVO}}$ & $\Delta_\text{variance}^{\text{TRVO}}$ & $\Delta_\text{SR}^{\text{TRVO}}$ & $\Delta_\text{J}^{\text{MVPI}}$ & $\Delta_\text{mean}^{\text{MVPI}}$ & $\Delta_\text{variance}^{\text{MVPI}}$ & $\Delta_\text{SR}^{\text{MVPI}}$ \\ \hline
InvertedP.& -4130\% & -5\% & $2\cdot 10^8$\% & -100\%& 0\% & 0\% & 0\% & 0\%\\ \hline
InvertedD.P.& -399\% & -13\% & 398\% & -61\%& 77\% & -31\% & -77\% & 44\%\\ \hline
HalfCheetah& 98\% & -84\% & -98\% & 0\%& 96\% & -66\% & -95\% & 60\%\\ \hline
Walker2d& 47\% & -71\% & -48\% & -60\%& 91\% & -65\% & -91\% & 16\%\\ \hline
Swimmer& -210\% & -23\% & 619\% & -71\%& -6\% & -7\% & -13\% & -1\%\\ \hline
Hopper& 64\% & -31\% & -64\% & 15\%& 90\% & -18\% & -89\% & 152\%\\ \hline
Reacher& -43\% & -2\% & 74\% & 23\%& 0\% & 1\% & 1\% & 1\%\\ \hline
Ant& 97\% & -94\% & -97\% & -65\%& 97\% & -63\% & -97\% & 101\%\\ \hline
\end{tabular}
\caption{\label{tab:normalized-2} 
Table~\ref{tab:normalized} with $\lambda=2$.
}
\end{table}

\begin{figure}[h]
\centering
\includegraphics[width=0.8\textwidth]{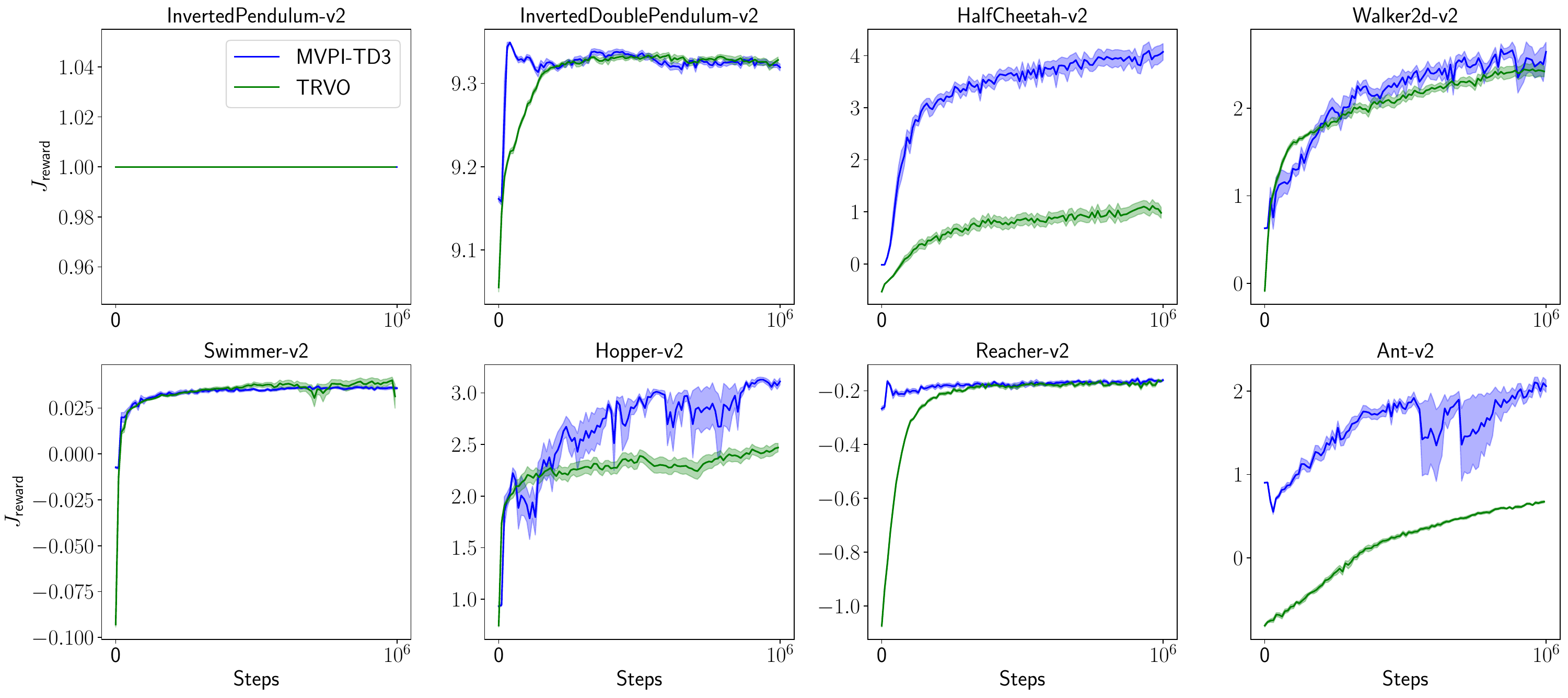}
\caption{\label{fig:mvpi-td3-per-step-reward-0.5} Comparing MVPI-TD3 and TRVO in terms of $J_\text{reward}$ with $\lambda = 0.5$}
\end{figure}

\begin{figure}[h]
\centering
\includegraphics[width=0.8\textwidth]{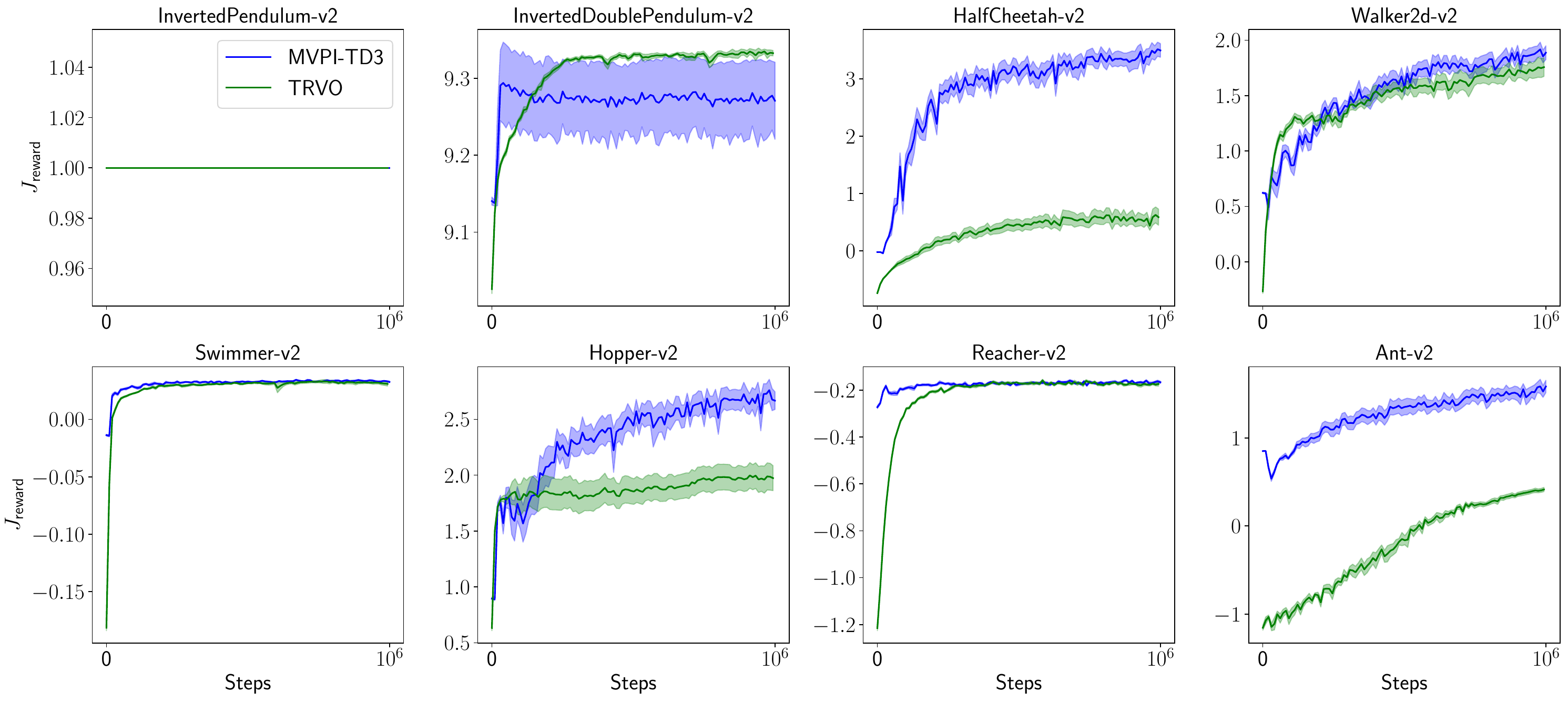}
\caption{\label{fig:mvpi-td3-per-step-reward-1} Comparing MVPI-TD3 and TRVO in terms of $J_\text{reward}$ with $\lambda = 1$}
\end{figure}

\begin{figure}[h]
\centering
\includegraphics[width=0.8\textwidth]{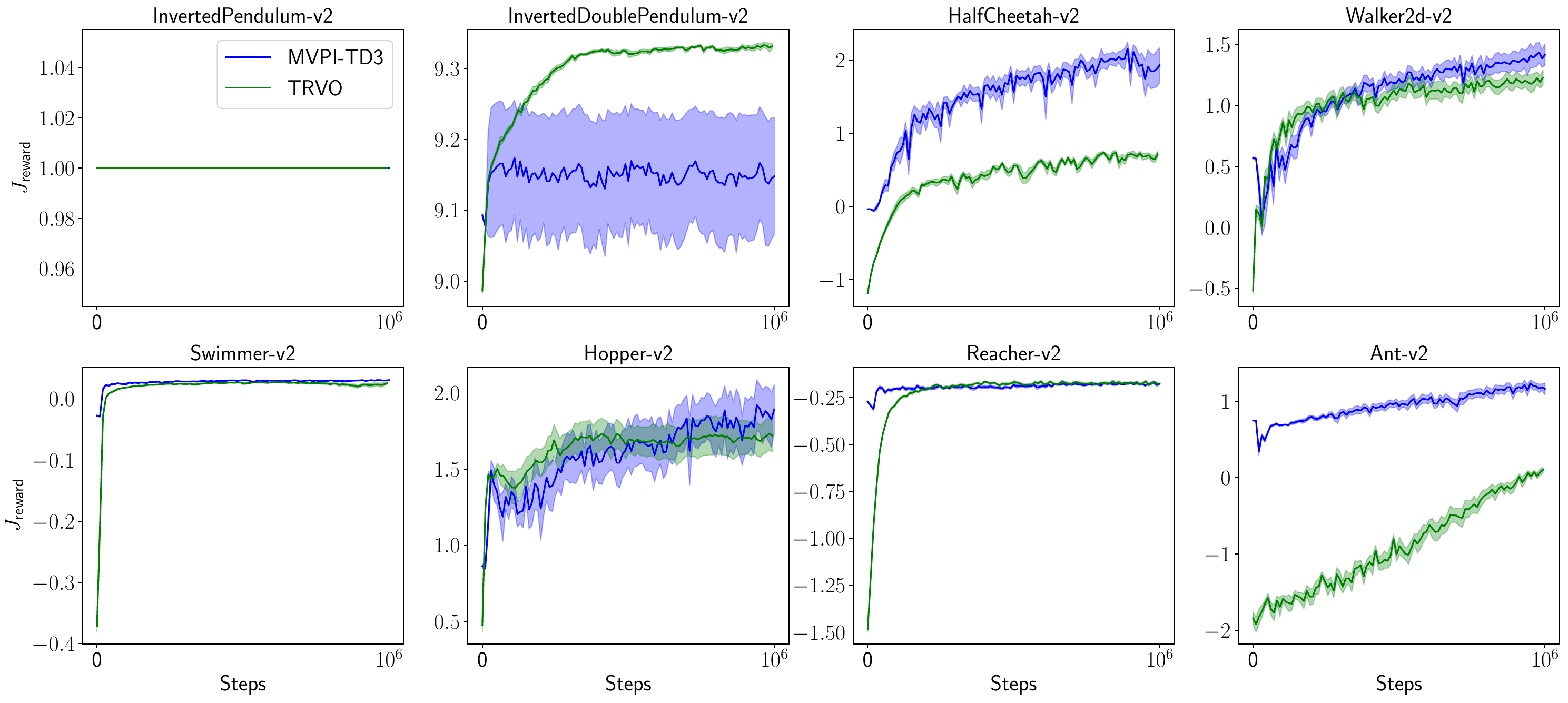}
\caption{\label{fig:mvpi-td3-per-step-reward-2} Comparing MVPI-TD3 and TRVO in terms of $J_\text{reward}$ with $\lambda = 2$}
\end{figure}

\end{document}